\documentclass{article}

\usepackage{times}
\usepackage{adjustbox}
\usepackage{tabularx,bigstrut}
\usepackage{graphicx} 
\usepackage{subfigure} 

\usepackage{natbib}

\usepackage{algorithm}
\usepackage{algorithmic}

\usepackage{hyperref}

\usepackage[accepted]{icml2016} 
\usepackage{hyperref}
\usepackage{url}
\usepackage{cleveref}
\usepackage{wrapfig}
\usepackage{amsmath,longtable,fancyhdr,booktabs,multirow,graphicx,float}
\usepackage{amssymb,amsthm}
\usepackage{color}
\usepackage{colortbl}
\usepackage{theoremref}

\newcommand{\yd}{y_{\mathrm{direct}}}

\newcommand{\mcal}{\mathcal}

\newcommand{\norm}[1]{\left\lVert#1\right\rVert}
\newcommand{\E}[1]{\mathbb{E}\left[#1\right]}

\renewcommand{\P}[1]{p \left[#1\right]}

\newtheorem{lemma}{Lemma}
\newtheorem{theorem}{Theorem}

\newtheorem{observation}{Observation}

\renewcommand{\P}{{\mathcal{P}}}
\newcommand{\N}{{\mathcal{N}}}
\newcommand{\D}{{\mathcal{D}}}
\newcommand{\X}{{\mathcal{X}}}
\newcommand{\Y}{{\mathcal{Y}}}
\newcommand{\be}{\begin{equation}}
\newcommand{\ee}{\end{equation}}

\definecolor{Gray}{gray}{0.85}
\definecolor{LightCyan}{rgb}{0.88,1,1}

\newcolumntype{a}{>{\columncolor{Gray}}c}
\newcolumntype{b}{>{\columncolor{white}}c}

\DeclareMathOperator*{\argmax}{arg\,max}

\makeatletter
\usepackage{xspace}
\def\@onedot{\ifx\@let@token.\else.\null\fi\xspace}
\DeclareRobustCommand\onedot{\futurelet\@let@token\@onedot}

\newcommand{\figref}[1]{Fig\onedot~\ref{#1}}
\newcommand{\equref}[1]{Eq\onedot~\eqref{#1}}
\newcommand{\secref}[1]{Sec\onedot~\ref{#1}}
\newcommand{\tabref}[1]{Tab\onedot~\ref{#1}}
\newcommand{\thmref}[1]{Theorem~\ref{#1}}

\newcommand{\lemref}[1]{Lemma~\ref{#1}}

\def\eg{\emph{e.g}\onedot} 
\def\ie{\emph{i.e}\onedot} 
 
\def\etc{\emph{etc}\onedot}

\icmltitlerunning{Training Deep Neural Networks via Direct Loss Minimization}

\begin{document} 

\twocolumn[
\icmltitle{Training Deep Neural Networks via Direct Loss Minimization}

\icmlauthor{Yang Song}{songyang12@mails.tsinghua.edu.cn}
\icmladdress{Dept. of Physics, Tsinghua University,
            Beijing 100084, China}
\icmlauthor{Alexander G. Schwing}{aschwing@cs.toronto.edu}
\icmlauthor{Richard S. Zemel}{zemel@cs.toronto.edu}
\icmlauthor{Raquel Urtasun}{urtasun@cs.toronto.edu}
\icmladdress{Dept. of Computer Science, University of Toronto,
             Toronto, Ontario M5S 2E4, Canada}

\icmlkeywords{structured prediction, neural networks, direct loss minimization}

\vskip 0.3in

]

\begin{abstract}

Supervised training of deep neural nets typically relies on minimizing
cross-entropy. However, in many domains, we are interested in performing well
on metrics specific to the application.  In this paper we propose a direct loss
minimization approach to train deep neural networks, which provably minimizes
the application-specific loss function. This is often non-trivial, since these
functions are neither smooth nor decomposable and thus are not amenable
to optimization with standard gradient-based methods. We demonstrate the
effectiveness of our approach in the context of maximizing average precision
for ranking problems. Towards this 
goal, we develop a novel
dynamic programming algorithm that can efficiently compute
the weight updates. Our approach proves superior to a variety of baselines in
the context of action classification and object detection, especially in the
presence of label noise.

\end{abstract}

\section{Introduction}

Standard supervised neural network training
involves computing the gradient of the loss function with respect to the parameters of the model, and therefore
requires  the loss function to be differentiable.
Many interesting loss functions are, however, {\it non-differentiable} with
respect to the output of the network.
Notable examples are functions
based on discrete outputs, as is common in labeling and ranking problems.
In many cases these losses are also {\it non-decomposable},
in that they cannot be expressed as simple sums over the output
units of the network.

In the context of structured prediction problems, in which the
output is multi-dimensional, researchers
have developed max-margin training methods that are capable of minimizing an
upper bound on non-decomposable loss functions. 
Standard learning in this paradigm involves changing the parameters 
such that the model assigns a higher score to the groundtruth output
than to any other output. This
is typically encoded by a  constraint, enforcing that the groundtruth score
should be higher than that of a selected, contrastive output. The latter is
defined as the result of inference performed using a modified score function
which combines the model score and the task loss, representing the metric
that we care about for the application. 
This modified scoring function encodes the fact that we  should penalize higher scoring configurations that are inferior in terms of the task
loss. Various efficient methods have been proposed for incorporating 
complex discrete loss functions into this max-margin approach
\citep{yue2007support,VolkovsICML2009,TarlowAISTATS2012,mohapatra2014efficient}. 
Importantly, however, this form of learning does not directly optimize the
task loss, but rather an upper bound.  

An alternative approach, frequently used in deep neural networks, is to train
with a surrogate loss that can be  easily optimized, \eg, cross-entropy
\citep{Lecun05lossfunctions,Bengio-et-al-2015-Book}. The problem of this
procedure is that for many application domains the task loss differs
significantly from the surrogate loss.

The seminal work of \citet{hazan2010direct} showed how to compute the gradient of complex non-differentiable loss functions when dealing with  linear models. 
In this paper we extend their theorem to   the non-linear case. 
This is important in practice as it provides us with a new learning algorithm
to train deep neural networks end-to-end to minimize the application specific
loss function.  
As shown in our experiments on action classification and object detection,
this is very beneficial, particularly when dealing with noisy labels.  


\section{Direct Loss Minimization for Neural Networks}
In this section we present a novel formulation for learning neural networks by
minimizing the task loss. Towards this goal, our first main result is a
theorem extending the direct loss minimization framework of \citep{hazan2010direct} to non-linear models. 

A neural network can be viewed as defining a 
 composite scoring function $F(x,y,w)$, which depends on the input data
 $x\in\X$, some parameters $w\in\mathbb{R}^A$, and the output  $y\in\Y$. 
 Inference is then performed by
 picking the output with maximal score, \ie: 
 
$$
y_w = \arg\max_{\hat y\in\Y} F(x,\hat y,w).
$$

Given a dataset of input-output pairs $\D = \{(x,y)\}$, a standard machine learning approach  is
to optimize the  parameters $w$ of the scoring function $F$ by optimizing cross-entropy. This is equivalent to maximizing the likelihood
of the data, where the probability over each output configuration is given by
the output of a softmax function, attached to the last layer of the network.  

However, in many practical applications, we want prediction to succeed in an application-specific metric. This metric is typically referred to  as the
{\it task loss},  $L(y,y_w) \geq 0$, which measures the compatibility between the annotated configuration $y$ and the prediction $y_w$.  
For learning, in this paper, we are thus interested in minimizing the task loss
\be
w^\ast = \arg\min_w \mathbb{E}\left[L(y,y_w)\right],
\label{eq:DLMProgram}
\ee
where $\E{\cdot}$ denotes an expectation taken over the underlying distribution behind the given dataset.

Solving this program is non-trivial, as many loss functions of interest are non-decomposable
and non-smooth, and thus are not amenable to gradient-based methods. 
Examples of such metrics include average precision (AP) from
information retrieval, intersection-over-union which is used in image
labeling, and normalized discounted cumulative gain (NDCG) which
is popular in ranking.  In general many metrics are discrete, are not
simple sums over the network outputs, and are not readily differentiable.

During training of a classifier it is hence common to employ a surrogate error metric, such as cross-entropy
or hinge-loss,
where one can directly compute the gradients with respect to the parameters. 
In the context of structured prediction models, several approaches have been
developed to effectively optimize the structured hinge loss with
non-decomposable task losses.
While these methods include the task loss in the objective, they are not
directly minimizing it, and hence these surrogate losses are at best
highly correlated with the desired metric. Finding efficient techniques to
directly minimize the metric of choice is therefore desirable. 

\citet{hazan2010direct}  showed  that it is  possible to asymptotically optimize the task-loss when the function $F$ is linear in the parameters \ie,  $F(x,y,w) = w^\top\phi(x,y)$. This work has produced encouraging results. For example, \citet{hazan2010direct} used this technique for phoneme-to-speech alignment on the TIMIT dataset optimizing the \textit{$\tau$-alignment loss} and the \textit{$\tau$-insensitive loss}, while \citet{KeshetInterspeech2011} illustrated applicability of the method to hidden Markov models for speech. Direct loss minimization was also shown to work well for inverse optimal control by \citet{doerr2015direct}.

The first contribution of our work is to generalize this theorem to arbitrary scoring functions, \ie, non-linear and non-convex functions. This allows us to derive a new training algorithm for deep neural networks which directly minimizes the task loss. 

\begin{theorem}[General Loss Gradient Theorem]
\label{thm:GLGT}
When given a finite set   $\mcal{Y}$,  a scoring function
 $F(x,y,w)$, a data distribution,
 as well as a task-loss $L(y,\hat y)$, 
then, under some mild regularity conditions (see the supplementary material for
details), the direct loss gradient has the following
 form:
 \begin{equation}
 \begin{split}
 &\nabla_w \E{L(y,y_w)} \\
 = &\pm \lim_{\epsilon \rightarrow 0} \frac{1}{\epsilon} \E{\nabla_w F(x,\yd,w) - \nabla_w F(x,y_w,w)}\label{eq:DirectLossGradient},
 \end{split}
 \end{equation} 
 with
 \begin{eqnarray}
 y_w &=& \argmax_{\hat y \in \Y} F(x,\hat y,w),\nonumber\\
 \yd &=& \argmax_{\hat y \in \Y} F(x,\hat y,w) \pm \epsilon L(y,\hat y) \label{eq:LossAugInf}.
 \end{eqnarray}
\end{theorem}
\begin{proof}
We refer the reader to the supplementary material for a formal proof of the theorem.
\end{proof}

\begin{figure*}
\centering
\fbox{
\begin{minipage}[c]{13.5cm}
{\bf Algorithm: Direct Loss Minimization for Deep Networks} 

Repeat until stopping criteria

\begin{enumerate}
\item Forward pass to compute $F(x,\hat y; w)$
\item Obtain $y_w$ and $\yd$ via inference and loss-augmented inference
\item Single backward pass via chain rule to obtain gradient $\nabla_w \E{L(y,y_w)}$
\item Update parameters using stepsize $\eta$: $w \leftarrow w - \eta\nabla_w \E{L(y,y_w)}$
\end{enumerate}

\end{minipage}
}
\caption{Our algorithm for direct loss minimization.}
\label{fig:GLGTalgo}
\end{figure*}

According to \thmref{thm:GLGT}, to obtain the gradient we need to find  the solution of two inference
problems. The first computes $y_w$, which is a standard inference
task, solved by  employing the forward propagation algorithm.  
The second inference problem is  prediction using a scoring function which is
perturbed by the task loss $L(y,\hat y)$. This is typically non-trivial
to solve, particularly when the task loss is not decomposable.
We borrow terminology from  the structured prediction literature,
where this perturbed inference problem is
commonly referred to as \emph{loss-augmented inference}~\citep{TsochantaridisJMLR2005,ChenICML2015}.  
In the following section we derive an efficient dynamic programming
algorithm to perform loss-augmented inference when the task loss is average precision.  

Note that loss-augmented inference, as specified in \equref{eq:LossAugInf},
can take the task loss into account in a positive or a negative way. Depending
on the sign, the gradient direction changes. \citet{hazan2010direct} provides
a nice intuition for the two different directions. The positive update
performs a step away from a worse configuration, while the negative direct
loss gradient encourages moves towards better outputs. This can be seen when
considering that maximization of the loss returns the worst output
configuration, while maximization of its negation returns the label with the
lowest loss. It is however an empirical question, whether the positive update
or the negative version performs better. \citet{hazan2010direct} reports
better results with   the negative update, while we find in our experiments
that the positive update is better when applied to deep neural networks. We
will provide intuition for this phenomenon in \secref{sec:exp}.

Note the relation between direct loss minimization and optimization of the structured hinge-loss. While we compute the gradient via the difference between the loss-augmented inference result and the prediction, structured hinge-loss requires computation of the difference between the loss-augmented inference solution and the ground truth. In addition, for any finite $\epsilon$, direct loss minimization is  sub-gradient descent of some corresponding ramp loss \citep{keshet2011generalization}.

In \figref{fig:GLGTalgo} we summarize the resulting learning algorithm, which
consists of the following four steps. First we use a standard forward pass to
evaluate $F$. We then perform inference and loss-augmented inference as
specified in \equref{eq:LossAugInf} to obtain the prediction $y_w$ and
$\yd$. We combine the predictions to obtain the gradient
$\nabla_w \E{L(y,y_w)}$ via a single backward pass which is then used to
update the parameters. For notational simplicity we omit details like
momentum-based gradient updates, the use of mini-batches, \etc, which are easily included. 


\section{Direct Loss Minimization for Average Precision}

\begin{figure*}
\centering
\fbox{
\begin{minipage}[c]{13.5cm}
{\bf Algorithm: AP loss-augmented inference} 

\begin{enumerate}
\item Set $h(1,0) = \mp\epsilon \frac{1}{|\mcal{P}|}$ and $h(0,1) = 0$
\item For $i = 1, \ldots, |\P|$, $j = 1, \ldots, |\N|$, recursively fill the matrix
\begin{align*}
h(i,j) = \max \begin{cases}
h(i-1,j) \mp  \epsilon \frac{1}{|\mcal{P}|}\frac{i}{i+j} + B(i,j), \quad i \geq 1, j \geq 0\\
h(i,j-1) + G(i,j),\quad i \geq 0, j \geq 1
\end{cases}
\end{align*}
\item Backtrack to obtain configuration $\hat{y}^*$
\end{enumerate}
\end{minipage}
}
\caption{Our algorithm for AP loss-augmented maximization or minimization.}
\label{fig:APDynAlgo}
\end{figure*}
In order to directly optimize the task-loss we are required to compute the gradient defined in \equref{eq:DirectLossGradient}. As mentioned above, we need to solve both the standard inference task as well as the loss-augmented inference problem given in \equref{eq:LossAugInf}. While the former is typically assumed to be solvable, the latter depends on $L$ and might be very complex  to solve, \eg, when the loss is not decomposable.

In this paper we consider ranking problems, where the desired task loss $L$ is average precision (AP), a concrete example of a non-decomposable and non-smooth target loss function.
For the linear setting, efficient algorithms for positive loss-augmented inference with AP loss were proposed by \citet{yue2007support} and \citet{mohapatra2014efficient}. Their results can be extended to the non-linear setting only in the positive case, where $\yd = \argmax_{\hat y \in \Y} F(x,\hat y,w) + \epsilon L(y,\hat y)$. For the negative setting inequalities required in their proof do not hold and thus their method is not applicable. In this section we propose a more general algorithm that can handle both cases with the same time complexity as \citep{yue2007support}, while being more intuitive to prove and understand.

Alternatives to optimizing average precision are methods such as RankNet \citep{burges2005learning}, LambdaRank \citep{quoc2007learning} and LambdaMART \citep{BurgesTR2010}. For an overview, we refer the reader to~\citet{BurgesTR2010} and references therein. Our goal here is simply to show direct loss minimization of AP as an example of our general framework. 

To compute the AP loss we are given a set of positive, \ie, relevant, and negative samples. The sample $x_i$ belongs to the positive class if $i\in\P = \{1, \ldots, |\P|\}$, and $x_i$ is part of the negative class if $i\in\N = \{|\P|+1, \ldots, |\P| + |\N|\}$.

We define the  output  to be composed of pairwise comparisons, with 
$y_{i,j} = 1$ if sample $i$ is ranked higher than sample $j$, $y_{i,i} = 0$, and $y_{i,j} = -1$ otherwise. We subsumed all these pairwise comparisons in $y = (\ldots, y_{i,j}, \ldots)$. Similarly $x = (x_1, \ldots, x_{N})$ contains all inputs, and $N= |\P| + |\N|$ refers to the total number of data points in the training set. 
In addition, we assume the ranking across all samples $y$ to be complete, \ie, consistent.
During inference we obtain a ranking by predicting scores $\phi(x_i,w)$ for all data samples $x_i$ which are easily sorted afterwards.

For learning we generalize the feature function defined by \citet{yue2007support} and \citet{mohapatra2014efficient} to a non-linear scoring function, using
$$
F(x,y,w) = \frac{1}{|\P||\N|} \sum_{i\in\P,j\in\N} y_{i,j}\left(\phi(x_i,w) - \phi(x_j,w)\right).
$$
where $\phi(x_i,w)$ is the output of the deep neural network when using the  $i$-th example as input. 

AP is unfortunately a non-decomposable loss function, \ie, it does not decompose into functions dependent only on  the individual $y_{i,j}$.  
To define the non-decomposable AP loss formally, we construct a vector $\hat p = \operatorname{rank}(\hat y)\in\{0,1\}^{|\P|+|\N|}$ by sorting the data points according to the ranking defined by the configuration $\hat y$. This vector contains a $1$ for each positive sample and a value $0$ for each negative element. 
In applications such as object detection, an example is  said to be positive if the intersection over union of its bounding box and the ground truth box is bigger than a certain threshold (typically 50\%).  
Using the $\operatorname{rank}$ operator we obtain the AP loss by comparing two vectors $p = \operatorname{rank}(y)$ and $\hat p = \operatorname{rank}(\hat y)$ via
\be
L_{\operatorname{AP}}(p,\hat p) = 1 - \frac{1}{|\P|}\sum_{j:\hat{p}_j=1} \operatorname{Prec@j},
\label{eq:APLoss}
\ee
where $\operatorname{Prec@j}$ is the percentage of relevant samples in the prediction $\hat p$ that are ranked above position $j$.

To solve the loss-augmented inference task we have to solve the following program
\be
\arg\max_{\hat y} F(x,\hat y,w) \pm \epsilon L_{\operatorname{AP}}(\operatorname{rank}(y),\operatorname{rank}(\hat y)).
\label{eq:LossAugInf2}
\ee


In the following we  derive a dynamic programming algorithm that can handle both the positive and negative case and has the same complexity  as \citep{yue2007support}. 
Towards this goal, we first note that Observation 1 of \citet{yue2007support} holds for both the positive and the negative case. For completeness, we repeat their observation here and adapt it to our notation.

\begin{observation}[\citet{yue2007support}]
Consider rankings which are constrained by fixing the relevance at each position in the ranking (\eg, the 3rd sample in the ranking must be relevant). Every ranking  satisfying the same set of constraints will have the same $L_{AP}$. If the positive samples are sorted by their scores in descending order, and the irrelevant samples are likewise sorted by their scores, then the interleaving of the two sorted lists satisfying the constraints will maximize \equref{eq:LossAugInf2} for that constrained set of rankings. 
\end{observation}

Observation~1 means that we only need to consider the interleaving of two sorted lists of $\mcal{P}$ and $\mcal{N}$ to solve the program given in \equref{eq:LossAugInf2}. From now on we therefore assume that the elements of $\mcal{P}$ and $\mcal{N}$ are sorted in descending order of their predicted score. 

We next assert the optimal substructure property of our problem. Let the restriction to subsets of $i$ positive and $j$ negative examples be given by $\mcal{P}_i = \{1, \ldots, i\}$ and $\mcal{N}_j = \{|\P|+1, \ldots, |\P| + j\}$. 
The cost function value obtained when restricting loss-augmented inference to the subsets can be computed as:
\be
\begin{split}
 &h(i,j)\\ =&\max_{\hat y}  \frac{1}{|\mcal{P}||\mcal{N}|} \sum_{m\in \mcal{P}_i}\sum_{n \in \mcal{N}_j} \hat y_{m,n}(\phi(x_{m},w) - \phi(x_{n},w))\\
 &\pm \epsilon L_{AP}^{i,j}(\operatorname{rank}(y),\operatorname{rank}(\hat y)),
 \end{split}
\label{eq:OptSub}
\ee
where $L_{AP}^{i,j}$ refers to the AP loss restricted to subsets of $i$ positive and $j$ negative elements.

\begin{lemma}\label{lem:ldp}
Suppose that $\operatorname{rank}(\hat y^\ast)$ is the optimal ranking for \equref{eq:OptSub} when restricted to $i$ positive and $j$ negative samples. Any of its sub-sequences starting at position 1 is then also an optimal ranking for the corresponding restricted sub-problem.
\end{lemma}

We provide the proof of this lemma in the supplementary material.
%
Based on \lemref{lem:ldp} we can construct Bellman equations to recursively fill in a matrix of size $\P\times\N$, resulting in an overall time complexity of $O(|\mcal{P}||\mcal{N}|)$. 
We can then obtain the optimal loss-augmented predicted ranking via back-tracking. 

The Bellman recursion computes the optimal cost function value of the sub-sequence containing data from $\P_i$ and $\N_j$, based on previously computed values as follows:
\begin{eqnarray*}
&&\hspace{-0.4cm}h(i,j) =\\
&&\hspace{-0.4cm} \max \begin{cases}
h(i-1,j) \mp  \epsilon \frac{1}{|\mcal{P}|}\frac{i}{i+j} + B(i,j) &\hspace{-0.2cm} i \geq 1, j \geq 0\\
h(i,j-1) + G(i,j) &\hspace{-0.2cm}i \geq 0, j \geq 1
\end{cases},
\end{eqnarray*}
with  initial conditions  $h(1,0) = \mp\epsilon \frac{1}{|\mcal{P}|}$ and $h(0,1) = 0$. Note that 
 we used the pre-computed matrices of scores
\begin{align*}
B(i,j) &=- \frac{1}{|\mcal{P}||\mcal{N}|}\sum_{k\in \mcal{N}_j} (\phi(x_i,w) - \phi(x_k,w))\\
G(i,j) &= \frac{1}{|\mcal{P}||\mcal{N}|}\sum_{k \in \mcal{P}_i} (\phi(x_k,w) - \phi(x_j,w)).
\end{align*}
Intuitively, the Bellman recursion considers two cases: (i) how does the maximum score $h(i,j)$ for the solution restricted to sequences $\P_i$ and $\N_j$ change if we add a new positive sample; (ii) how does the maximum score $h(i,j)$ change when adding a new negative sample. In both cases we need to add the respective scores $B(i,j)$ or $G(i,j)$. When adding a positive sample we additionally need to consider a contribution from the precision which contains $i$ positive elements from a total of $i+j$ elements.

The matrices $B(i,j)$ and $G(i,j)$ store the additional contribution to \equref{eq:OptSub} obtained when adding a positive or a negative sample respectively. 
They can be efficiently  computed ahead of time using the recursion
\begin{align*}
B(i,j) &= B(i,j-1) - \frac{1}{|\mcal{P}||\mcal{N}|}(\phi(x_i,w) - \phi(x_j, w)) \\
G(i,j) &= G(i-1,j) + \frac{1}{|\mcal{P}||\mcal{N}|}(\phi(x_i,w) - \phi(x_j, w)),\\
\end{align*}
with initial conditions $B(i,0) = 0$, $G(0,j) = 0$. 

We summarize our dynamic programming  algorithm for AP loss-augmented inference in \figref{fig:APDynAlgo}.
Note that after having completed the matrix $h(i,j)$ we can backtrack to obtain the best loss-augmented ranking. We hence presented a general solution to solve positive and negative loss-augmented inference defined in \equref{eq:LossAugInf} for the AP loss given in \equref{eq:APLoss}. 


\section{Experimental Evaluation}
\label{sec:exp}

To evaluate the performance of our approach we perform experiments on both synthetic  and real datasets. We compare the positive and negative version of our direct loss minimization approach to a diverse set of baselines. 

\subsection{Synthetic data}
\paragraph{Dataset:} We generate synthetic data via a neural network which
assigns scalar scores to input vectors.
The neural network consists of four layers of rectified linear units with 
 parameters  randomly drawn
from independent  Gaussians with zero mean and unit variance.
The input for a training example is drawn 
from a 10 dimensional  Gaussian,
and the output produced by the network is its score.
We generate 20,000 examples, and sort them in descending order based on
their scores. The top $20\%$
of the samples are assigned to the positive set $\P$ and the remaining examples to the negative set. We then  randomly divide
the generated data into a training set containing 10,000 elements and a test
set containing the rest. We  compare various loss functions in terms
of their ability to train the network to effectively reproduce the original
scoring function.
To ensure that we do not suffer from model
mis-specification we employ the same network structure when training the
parameters from random initializations.
This synthetic experiment provides a good test environment to compare
these training methods, as the mapping from input to scores is fairly complex
yet deterministic.

\paragraph{Algorithms:}
We evaluate the positive and negative versions of our direct loss minimization when using two different task losses: AP and 0-1 loss. 
For the latter we predict for each sample $x_i$ whether it is a member of the set $\P$ or whether it is part of the set $\N$.
We named these algorithms, ``pos-AP,'' ``neg-AP,'' ``pos-01,'' and ``neg-01.''
We also evaluate training the network using hinge loss, when employing AP and
0-1 loss as the task losses. We called these baselines ``hinge-AP'' and
``hinge-01.'' 
Note that ``hinge-AP'' is equivalent to the approach of Yue et al. \citet{yue2007support}.
We use the perceptron updates as additional baselines, which we call ``per-AP'' and ``per-01.'' Finally, the last baseline uses maximum-likelihood (\ie, cross entropy) to train the network. We call this approach ``x-ent.'' 
The parameters of all algorithms are individually determined via grid search to produce the best AP on the training set. 

\begin{figure*}
\centering
\subfigure[]{
\label{fig:AP}\includegraphics[width=6.44cm,height=6.44cm]{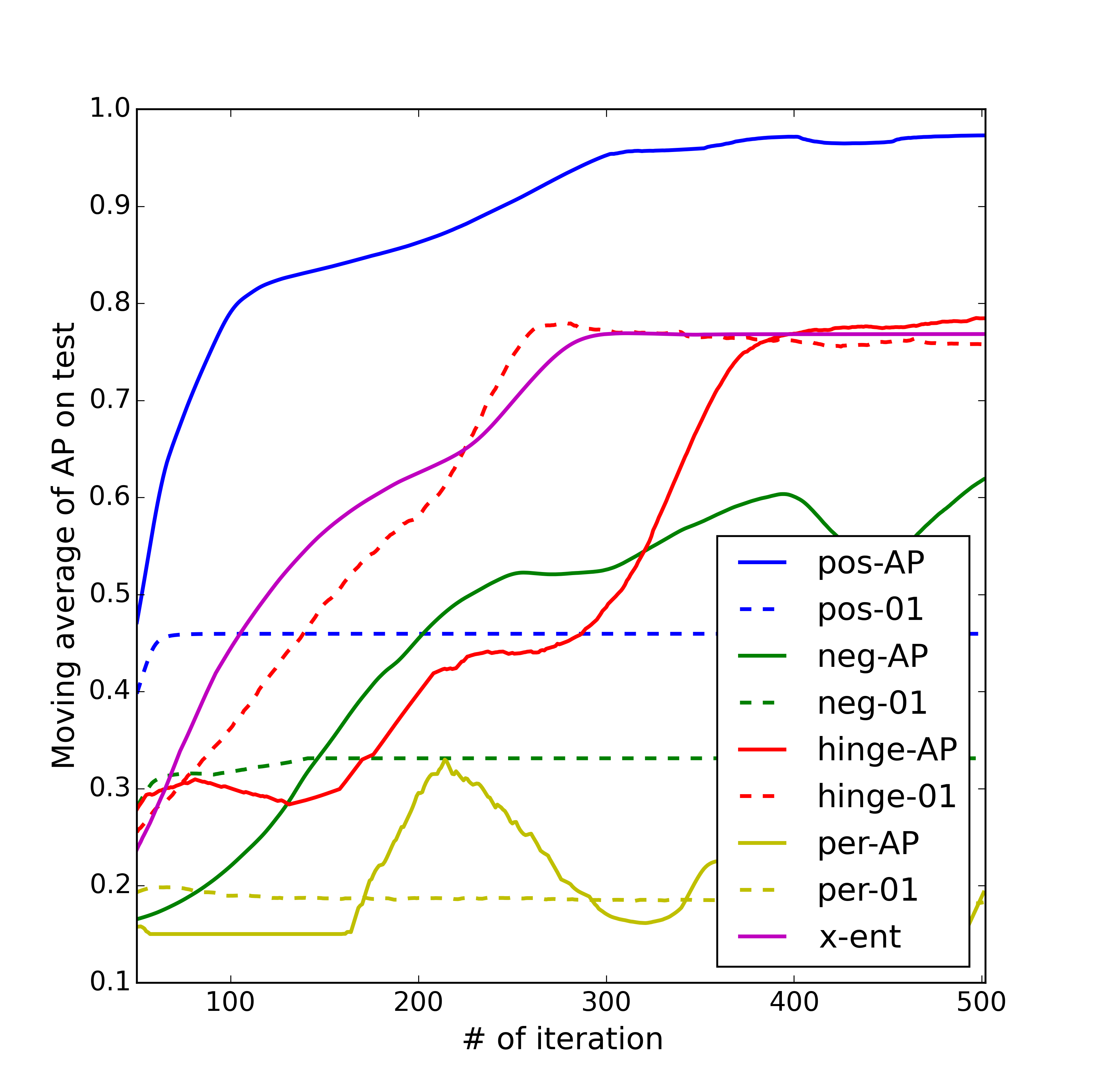}
}
\subfigure[]{
\label{fig:synflip}\includegraphics[width=6.44cm,height=6.44cm]{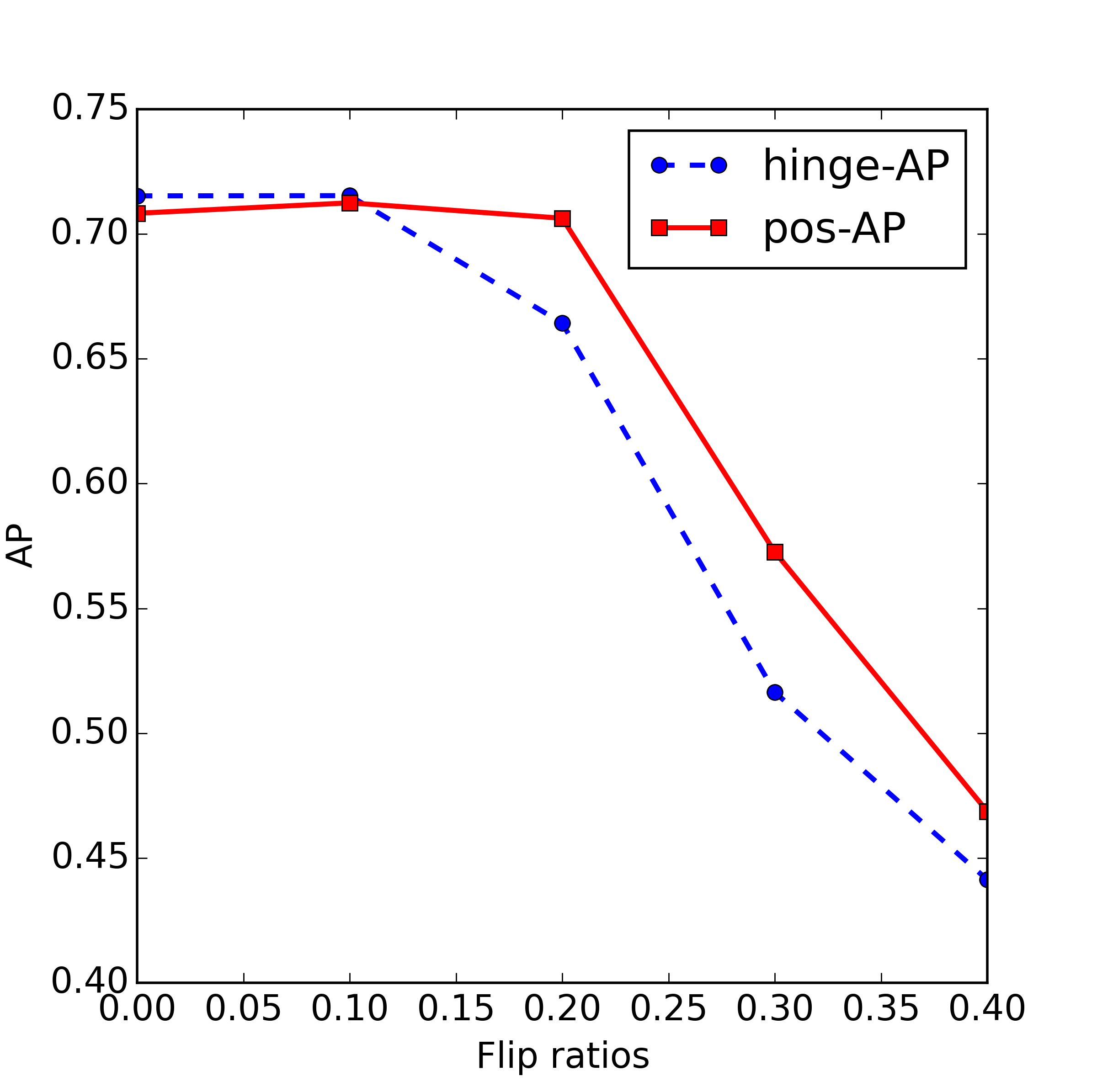}
}
\caption{Experiments on synthetic data: 
\subref{fig:AP} Average Precision (AP) on the test set as a function of the number of
iterations (best view in color). 
\subref{fig:synflip} The robustness of pos-AP compared to hinge-AP. } 
\label{fig:APall}
\end{figure*}

\paragraph{Results:}
 \figref{fig:APall} shows AP results on the test set.
We observe that the perceptron update does not perform well, since it does not
take the task loss into account and has convergence issues. 
Contrary to the claim in \citep{hazan2010direct} for linear models, the negative update for direct loss  is not competitive in our setting (\ie, non-linear models). This might be due to the fact
that it tends to overly correct the classifier in noisy situations. Note the
resemblance of the perceptron method and the negative direct loss minimization
update: they are both trying to move ``towards better.'' Therefore we expect
the negative update to behave similarly to the perceptron, \ie, more
vulnerable to noisy data. This resemblance also explains their
strong performance on the TIMIT dataset \citep{cheng2009matrix,hazan2010direct},
which is relatively easy and can be handled well with linear models. For the
same reason, we expect the negative update to lead to better performance in
less noisy situations, \eg, if the data is nearly linearly separable. To
further examine this hypothesis we provide additional experimental results in
the supplementary material.

The direct loss minimization of 0-1 loss does not work well. It is likely that the sharp changes of the 0-1 loss result in a ragged energy landscape that is hard to optimize, while smoothing
as performed by AP loss or surrogate costs helps in this setting. 

The hinge-AP algorithm of 
\citet{yue2007support} performs slightly better than other algorithms
based on 0-1 loss. This shows 
that taking the AP loss into consideration helps improve the performance
measured by AP. It is also important to note that when employing positive
updates our direct loss minimization outperforms all baselines by a large
margin. Note the resemblance between the structural SVM update and positive direct
loss minimization: they are both moving ``away from worse.'' 
We believe that hinge loss does not deal with label noise well as the update depends on the ground truth more than the direct loss update. Another possible reason is that the hinge loss upper bound is looser in this case. Taking the 0-1 loss as an example, the hinge loss of each outlier is much larger than the cost measured by the 0-1 loss (which is at most 1) in noisy case. This illustrates why the gap between hinge loss and target loss is large in noisy situations in general.

To further examine the performance of pos-AP and hinge-AP, we performed
another small scale synthetic experiment. We hypothesize that when the
data contains outliers, and more generally the labels are noisy, 
hinge loss will become a worse approximation.
We control the noise level in the data as a way of testing this hypothesis.
We randomly generate 1000
10-dimensional data from $\mathcal{N}(0,10)$. Datum sample $x$ is assigned to
be positive when $\norm{x}_2^2 > 1200$ and negative when $\norm{x}_2^2 <
1000$. The noise is incorporated by randomly flipping a fixed percentage of labels. We use the same neural network structure for this task, tune the
parameters on the training set, and report their results on the independent
test set. As shown in \figref{fig:synflip}  our method, pos-AP, is more robust to noise. 

Based on the results obtained from the synthetic experiments,
during the experimental evaluation on real datasets
we focus on comparing the positive non-linear direct
loss minimization (pos-AP) to the strong baseline of hinge-AP
\citet{yue2007support},
as well as the standard approach of training based on cross-entropy.

\subsection{Action classification task}
\begin{table*}
\begin{center}
\begin{adjustbox}{max width=\textwidth}
\setlength{\tabcolsep}{4pt}
\begin{tabular}{c|ccc||ccc||ccc||ccc||ccc}
noise level & \multicolumn{3}{c||}{0} & \multicolumn{3}{c||}{10\%} & \multicolumn{3}{c||}{20\%} & \multicolumn{3}{c||}{30\%} & \multicolumn{3}{c}{40\%}\bigstrut\\
\hline
method & x-ent & hinge-AP & pos-AP & x-ent & hinge-AP & pos-AP & x-ent & hinge-AP & pos-AP & x-ent & hinge-AP & pos-AP & x-ent & hinge-AP & pos-AP \bigstrut \\
\hline
jumping & 76.6 & 77.0 & \textbf{77.3} & 68.7&  64.0&	\textbf{74.0}&	51.6&	44.9&	\textbf{65.1}&	42.3&	\textbf{42.4}&	36.3&	22.8&	27.2&	\textbf{52.0} \bigstrut[t]\\
phoning & 45.4&	\textbf{46.2}&	44.8&	36.3&	31.2&	\textbf{39.2}&	25.7&	22.1&	\textbf{35.4}&	11.8&	10.8&	\textbf{15.4}&	9.5&	10.2&	\textbf{16.8}\\
playing instrument
 & 72.6&	\textbf{74.0}&	\textbf{74.0}&	67.9&	67.6&	\textbf{71.4}&	60.6&	57.3&	\textbf{69.9}&	38.3&	40.0&	\textbf{62.8}&	25.1&	15.9&	\textbf{60.6}\\
reading
 & 50.3&	\textbf{50.6}&	\textbf{50.6}&	40.1&	36.8&	\textbf{43.4}&	27.0&	22.1&	\textbf{40.1}&	\textbf{17.9}&	15.9&	17.3&	13.8&	\textbf{15.8}&	9.6\\
riding bike
 & 92.3&	92.9&	\textbf{93.1}&	84.9&	88.2&	\textbf{90.1}&	72.9&	73.3&	\textbf{88.2}&	54.1&	40.6&	\textbf{79.3}&	\textbf{32.9}&	18.3&	22.3\\
riding horse
 & 89.1&	92.0&	\textbf{92.4}&	78.4&	82.8&	\textbf{84.9}&	70.2&	77.8&	\textbf{79.4}&	45.7&	48.0&	\textbf{69.7}&	25.1&	28.4&	\textbf{53.2}\\
running
 & 82.3&	84.0&	\textbf{84.8}&	\textbf{77.9}&	76.9&	76.2&	64.6&	57.3&	\textbf{75.8}&	40.5&	44.0&	\textbf{71.3}&	17.7&	9.3&	\textbf{29.8}\\
taking photo
 & 41.2&	\textbf{45.8}&	43.2&	33.0&	\textbf{40.7}&	33.0&	19.0&	21.7&	\textbf{22.4}&	15.9&	12.1&	\textbf{19.5}&	11.0&	\textbf{11.3}&	8.4\\
using computer
 & \textbf{68.7}&	68.3&	66.5&	57.6&	59.3&	\textbf{60.2}&	42.3&	45.1&	\textbf{56.0}&	21.7&	21.5&	\textbf{41.4}&	13.2&	\textbf{15.1}&	11.1\\
walking
 & 61.3&	68.0&	\textbf{68.2}&	51.9&	\textbf{53.3}&	52.2&	39.0&	35.1&	\textbf{46.1}&	29.1&	24.6&	\textbf{46.3}&	11.2&	16.2&	\textbf{32.6}\\
\hline
mean & 68.0&	\textbf{69.9}&	69.5&	59.7&	60.1&	\textbf{62.5}&	47.3&	45.7&	\textbf{57.8}&	31.7&	30.0&	\textbf{45.9}&	18.2&	16.8&	\textbf{29.6}\bigstrut[t]
\end{tabular}
\end{adjustbox}
\end{center}
\caption{Comparison of our direct AP loss minimization approach to two strong
  baselines, which utilize surrogate loss functions, on the action
  classification task. Each method is evaluated for various amounts of label
  noise in the test dataset. 
  } \label{tab:ac}
\end{table*}
\paragraph{Dataset:} In the next experiment we use the PASCAL VOC2012 action
classification dataset provided by \citet{everingham2014pascal}. The dataset
contains 4588 images and 6278 ``trainval'' person bounding boxes.  For each of
the 10 target classes, we divide the trainval dataset  into equal-sized
training, validation and test sets. We tuned the learning rate, regularization
weight, and $\epsilon$ for all the algorithms based on their performance on
the validation dataset, and report the results on the test set. 
For all algorithms we used the entire available training set in a single batch and performed 300 iterations. 

\paragraph{Algorithms:}
We train our non-linear direct loss minimization as well as all the baselines
individually for each class. As baselines we again use a deep network trained with
cross entropy  and also consider the structured SVM method
proposed by \citet{yue2007support}. 
The deep network used in these experiments follows the architecture
of~\citet{krizhevsky2012imagenet}, with the top dimension adjusted to a single
output. We initialize 
the parameters using the weights trained on
ILSVRC2012~\citep{ILSVRC15}. Inspired by the RCNN ~\citep{girshick2014rich},
we cropped the regions of each image with a padding of 16 pixels and
interpolated them to a size of $227\times227\times3$ to fit the input data
dimension of the network. All the algorithms we compare to as well as our
approach use raw pixels as input.

\paragraph{Results:} Intuitively we expect direct loss minimization to outperform surrogate loss
functions whenever there is a significant number of outliers in the data. To
evaluate this hypothesis, we conduct experiments by randomly flipping  a fixed number of labels.  
Our experiments shown in \tabref{tab:ac} and \figref{fig:flip} confirm our intuitions, and direct loss works much better than hinge loss in the presence of label noise. When there is no noise, both algorithms perform similarly. 

\begin{figure}[t]
\centering
\includegraphics[width=\columnwidth]{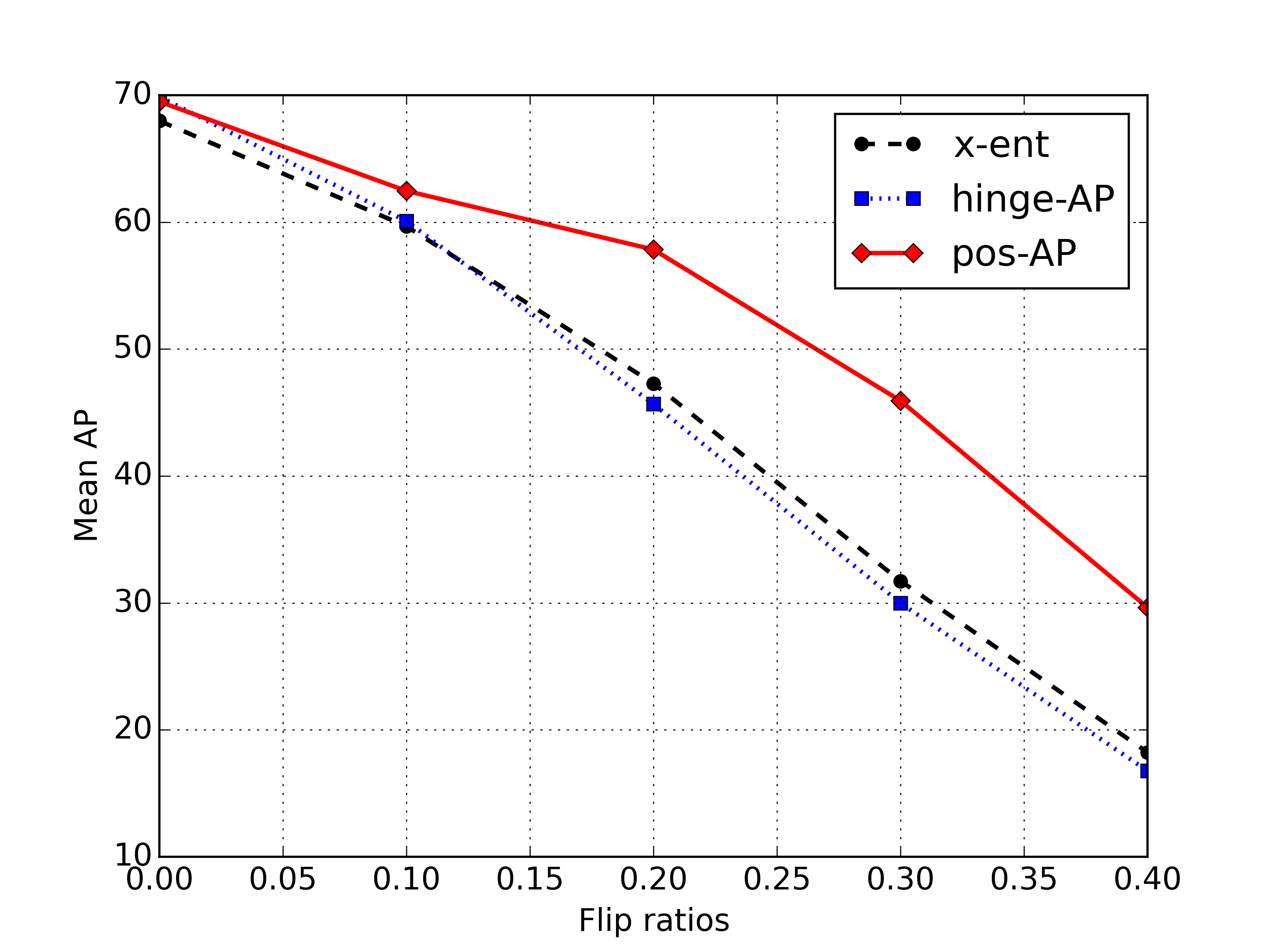}
\caption{This figure summarizes the effect of label noise in the action
  classification task, by showing the mean AP on the test set for different
  proportions of flipped labels. We compare the performance of the strongest
  baseline method, the hinge-loss trained network which uses AP  as the task
  loss, versus our direct loss method, pos-AP.}  
\label{fig:flip}
\end{figure}

\subsection{Object detection task}
\begin{table*}
\setlength{\tabcolsep}{4pt}
\centering
\begin{adjustbox}{max width=\textwidth}
\begin{tabular}{l|cccccccccccccccccccc|c}
 & \rotatebox{90}{aeroplane} & \rotatebox{90}{bicycle} & \rotatebox{90}{bird} & \rotatebox{90}{boat} & \rotatebox{90}{bottle} & \rotatebox{90}{bus} & \rotatebox{90}{car} & \rotatebox{90}{cat} & \rotatebox{90}{chair} & \rotatebox{90}{cow} & \rotatebox{90}{diningtable}
& \rotatebox{90}{dog} & \rotatebox{90}{horse} & \rotatebox{90}{motorbike} & \rotatebox{90}{person} & \rotatebox{90}{pottedplant} & \rotatebox{90}{sheep} & \rotatebox{90}{sofa} & \rotatebox{90}{train} & \rotatebox{90}{tvmonitor} & \rotatebox{90}{mean} \bigstrut[t]\\
\hline
x-ent (0 label noise) & 63.8 & \textbf{61.0} & 42.6 & 30.7 & 23.5 & 63.2 & 51.7 & 58.5 & 20.1 & 37.0 & 32.0 & 52.8 & 50.8 & 62.5 & 50.1 & 23.5 & 48.3 & 33.1 & 48.5 & 57.4 & 45.6 \bigstrut[t]\\
hinge-AP (0 label noise) & \textbf{67.5} & 60.6 & 43.6 & 30.8 & 25.3 & 64.5 & \textbf{54.9} & \textbf{64.4} & 21.9 & 34.5 & 34.2 & 57.0 & 48.8 & \textbf{63.9} & \textbf{56.3}& 25.1 & 49.6 & 37.4 & \textbf{54.3} & 57.3& 47.6 \\
pos-AP (0 label noise) & 65.1 & 59.8 & \textbf{43.7} & \textbf{31.4} & \textbf{27.7} & \textbf{64.6} & 53.1 & 63.7 & \textbf{25.6} & \textbf{40.2} & \textbf{36.2} & \textbf{58.1} & \textbf{52.8} & 63.6 & 56.2 & \textbf{28.1} & \textbf{50.0} & \textbf{38.9} & 50.0 & \textbf{61.3} & \textbf{48.5} \\
\hline\hline
\bigstrut hinge-AP (20\% label noise) & 0.0 & 0.0 & 0.0 & 0.0 & 0.0 & 0.0 & 0.0 & 0.0 & 0.0 & 0.0 & 0.0 & 0.0 & 0.0 & 0.0 & 0.0& 0.0 & 0.0 & 0.0 & 0.0 & 0.0& 0.0 \\
pos-AP (20\% label noise) & \textbf{52.8} & \textbf{54.0} & \textbf{33.6} & \textbf{20.9} &\textbf{20.0} & \textbf{50.6} & \textbf{45.9} & \textbf{55.7}& \textbf{23.1} & \textbf{26.4} & \textbf{35.2} & \textbf{47.2} & \textbf{39.7} & \textbf{54.2} & \textbf{53.3} & \textbf{22.5} & \textbf{42.6} & \textbf{32.5} & \textbf{40.5} & \textbf{55.0} & \textbf{40.3}
\end{tabular}
\end{adjustbox}
\caption{Comparison of our direct AP loss minimization approach to surrogate loss function optimization on the object detection task.}
\label{tab:td}
\end{table*}

\paragraph{Dataset:}

For object detection we use the PASCAL VOC2012 object detection dataset
collected by~\citet{everingham2014pascal}. The dataset  contains 5717 images
for training, 5823 images for validation and 10991 images for test. For each image, we use the fast mode of selective search by
\citet{UijlingsIJCV2013} to produce around 2000 bounding boxes. We train
algorithms on the training set and report results on the validation set. 

\paragraph{Algorithms:}
On this dataset we follow the RCNN  paradigm ~\citep{girshick2014rich}.
We adjust the dimension of the top layer of the
network~\citep{krizhevsky2012imagenet} to be one and fine-tune  using weights
pre-trained on ILSVRC2012~\citep{ILSVRC15}. We train direct loss minimization
for all 20 classes separately. In contrast to the action classification task,
we cannot calculate the overall AP in each iteration, due to the large
number of  bounding boxes. Instead, we use the AP on each mini-batch to
approximate the overall AP. We find that using a batch size of 512 balances
computational complexity and performance, though using a larger batch size
(such as 2048)  generally results in better performance. For our final
results, we use a learning rate of 0.1, a regularization parameter of $1\cdot
10^{-7}$, and $\epsilon = 0.1$ for all classes. 

As  baselines, we evaluate a network which uses cross-entropy and is trained
separately for each class. Again, the network structure was
chosen to be identical and we use the parameters provided by~\citet{krizhevsky2012imagenet}
for initialization. In addition, we consider the structured SVM algorithm,
which optimizes a surrogate of the AP loss. This structured SVM was trained
using the same batch size as our direct loss minimization. We use a learning
rate of 1, and a regularization parameter of $1\cdot 10^{-7}$ for all
classes. As usual, we compare hinge-AP and pos-AP in the presence of 20\% label noise.
\paragraph{Results:}
 \tabref{tab:td} shows 
competitive results of
 stochastic direct loss minimization, outperforming the strongest baseline by
 $0.9$. Our direct loss minimization performs better than
hinge loss in this case. This is because the data for training detectors
 are slightly noisier compared to the action classification task, which
is likely due to
 the common method of data augmentation based on intersection-over-union thresholds. To our astonishment, it becomes so hard for hinge-AP to learn well with noise in this detection task that it barely learns anything, while pos-AP only suffers from a reasonable decrease.

\section{Conclusion}

In this paper we have proposed a direct loss minimization approach to train
deep neural networks. We have demonstrated the effectiveness of our approach
in the context of maximizing average precision for ranking problems. This
involves minimizing a non-smooth and non-decomposable loss. Towards this goal
we have proposed a dynamic programming algorithm that can efficiently compute
the weight updates. Our experiments showed that  this is beneficial when
compared to a large variety of baselines in the context of action
classification and object detection, particularly in the presence of noisy
labels. 
In the future, we plan to investigate direct loss minimization in the context
of other non-decomposable losses, such as intersection over union for semantic
segmentation and shortest-path predictions in graphs.

\section*{Acknowledgments}
YS would like to thank the Department of Physics, Tsinghua University for providing financial support for his stay in Toronto and travel to ICML. We thank David A. McAllester for discussions and all the reviewers for helpful suggestions. This work was partially supported  by ONR Grant N00014-14-1-0232, and a Google researcher award.

\bibliography{icml2016_conference}
\bibliographystyle{icml2016}

\end{document}